\documentclass{article} % For LaTeX2e
\usepackage{fullpage}
\title{Learning Simple Auctions}
\usepackage{times}
\usepackage{natbib}
 \author{Jamie Morgenstern% \Email{jamiemmt@cs.cmu.edu}\\
\\
Tim Roughgarden
 }

\usepackage{pbox}
\usepackage{url}
\usepackage{algorithmicx}
\usepackage{enumitem}
\usepackage{amsmath}
\usepackage{amssymb,amsfonts}
\usepackage{theorem}
\usepackage{xspace}
\usepackage{url}

\newcommand{\argmax}{\textrm{argmax}}

\renewcommand{\v}{{\mathbf {v}}}

\newcommand{\V}{\mathcal{V}}
\newcommand{\Y}{\mathcal{Y}}

\newcommand{\com}{\textbf{compress}\xspace}
\newcommand{\decom}{\textbf{decompress}\xspace}
\newcommand{\PAC}{\ensuremath{\textsc{PAC}}\xspace}
\newcommand{\opt}{\ensuremath{\textsc{OPT}}\xspace}
\newcommand{\sam}{m}%sample size
\newcommand{\nbuy}{n}%number of buyers
\newtheorem{theorem}                            {Theorem}[section]

\newtheorem{corollary}          [theorem]       {Corollary}

{\theorembodyfont{\rmfamily} \newtheorem{definition}
[theorem]       {Definition}}
{\theorembodyfont{\rmfamily} }
{\theorembodyfont{\rmfamily} \newtheorem{remark}                [theorem]
{Remark}}
{\theorembodyfont{\rmfamily} \newtheorem{example}               [theorem]
{Example}}
{\theorembodyfont{\rmfamily} }
\theoremstyle{break}
{\theorembodyfont{\rmfamily} }

\newcommand{\D}{\mathcal{D}}
\newcommand{\R}{\mathbb{R}}

\newcommand{\F}{\mathcal{F}}

\newcommand{\VC}{\mathcal{VC}}
\newcommand{\err}{\textrm{err}}

\newcommand{\pd}{\mathcal{PD}}
\newcommand{\X}{\mathcal{X}}
\newcommand{\A}{\mathcal{A}}

\newcommand{\f}{f}
\newcommand{\E}{\mathbb{E}}
\newcommand{\I}{\mathbb{I}}

\newtheorem{observation}{Observation}
\newcommand{\maxval}{H}

\newenvironment{proof}{\noindent {\em {Proof:}}}{$\blacksquare$\vskip \belowdisplayskip}

\newenvironment{prevproof}[2]{\noindent {\em {Proof of                          
{#1}~\ref{#2}:}}}{$\blacksquare$\vskip \belowdisplayskip}

\newcommand{\rev}{\textsc{Rev}\xspace}

\renewcommand{\k}{\kappa\xspace}

\newcommand{\K}{K\xspace}

\newcommand{\p}{{\bf{p}}}

\newcommand{\B}{\mathcal{B}}

\begin{document}

\maketitle

\begin{abstract}
  We present a general framework for proving polynomial sample
  complexity bounds for the problem of learning from samples the best
  auction in a class of ``simple'' auctions.  Our framework captures
  all of the most prominent examples of ``simple'' auctions, including
  anonymous and non-anonymous item and bundle pricings, with either a
  single or multiple buyers. The technique we propose is to break the
  analysis of auctions into two natural pieces. First, one shows that
  the set of allocation rules have large amounts of structure; second,
  fixing an allocation on a sample, one shows that the set of auctions
  agreeing with this allocation on that sample have revenue functions
  with low dimensionality. Our results effectively imply that whenever
  it's possible to compute a near-optimal simple auction with a known
  prior, it is also possible to compute such an auction with an
  unknown prior (given a polynomial number of samples).
% We give
% a general framework for proving sample
%   complexity bounds for learning a revenue-maximizing simple auction
%   from sampled multi-parameter valuations. We propose a two-step
%   technique to bound the \emph{pseudo-dimension} (a real-valued analog
%   of the VC dimension) of a class of well-structured auctions. First,
%   one shows the auctions' allocation rules are linearly separable and
%   can therefore label a sample with at most $m^\alpha$ distinct
%   allocations. Second, one fixes some output of the allocation
%   function on the sample, and shows the \emph{revenue} functions for
%   all auctions consistent with that allocation function on the sample
%   have pseudo-dimension $d$. These two together imply the entire class
%   will have pseudo-dimension $\tilde{O}(d \alpha)$.  We use this
%   framework to analyze the sample complexity of many simple auction
%   classes from the mechanism design literature.
\end{abstract}

%%% Local Variables:
%%% mode: latex
%%% TeX-master: "paper"
%%% End:

\section{Introduction}\label{sec:intro}

The standard economic approach for designing revenue-maximizing
auctions assumes all information unknown to the designer is drawn from
some prior distribution, about which the designer has perfect
information. With this ``perfect'' prior in hand, the designer
fine-tunes an auction to optimize for her \emph{expected} revenue over
draws of the unknown information from the prior. While this model
allows for quite strong results relating her chosen auction's revenue
to the optimal revenue, three related difficulties arise from using
this design pattern in practice. First, for any particular setting, it
is unlikely that the designer could actually formulate a perfect prior
over the market's hidden information. Second, if the market designer
has an imperfect prior, it is possible that her optimal auction has
overfit to this prior and will have very poor revenue when run on the
(similar) true prior. Finally, the optimal auction for a particular
prior can be quite complicated and unintuitive.

These obstacles can be addressed in a rigorous manner by designing
auctions as a function of several \emph{samples} of the unknown data
(usually, buyers' valuations) drawn from an unknown distribution, with
the knowledge that our goal is to earn high revenue on a fresh draw
from the same distribution. It is reasonable to expect that
experienced sellers have previous records of the bids made by previous
participants in their market. Moreover, if an auction is guaranteed
perform well on future draws from a distribution to which it only had
sample access, it will have strong generalization properties if its
sample size was sufficiently large.

How many samples are necessary to achieve such a guarantee?  The
answer depends upon the complexity of the set of auctions the seller
might select.  The more complex the class of auctions, the lower the
class's ``representation error'' (the higher the revenue the seller
might be able to extract); on the other hand, a more complex class of
auctions will have higher ``generalization error'' (loss in revenue
from optimizing over the sample rather than the true prior) for a
fixed sample size.

The mechanism design community has placed ``simplicity'' as a design
goal for its own sake for single-parameter~\citep{hartline2009simple}
and
multi-parameter~\citep{chawla2007algorithmic,chawla2010multi,babaioffadditive,
  rubensteinsubadditive} auctions.\footnote{A canonical
  ``single-parameter'' problem is a single-item auction --- each
  bidder either ``wins'' or ``loses''.  A canonical
  ``multi-parameter'' problem is a multi-item auction --- with $k$
  items, each bidder faces $2^k$ different possibilities.
  Multi-parameter problems are well known to be much more ill-behaved
  than single-parameter problems.  For example, there is no general
  multi-parameter analog of Myerson's (single-parameter) theory of
  revenue-maximizing auctions \citet{myerson1981optimal}.}  Recent
work~\citep{morgenstern2015pseudo} proposed the use of a class's
pseudo-dimension as a formal notion of simplicity for single-parameter
auctions, and proves for general single-parameter settings there exist
classes of auctions with small representation error (which contain a
nearly-optimal auction) \emph{and} have small pseudo-dimension or
generalization error (a polynomial-sized sample suffices to learn a
nearly-optimal auction from that class).

In this work, we give a general framework for bounding the
pseudo-dimension of classes of multi-parameter auctions.  Our results
imply polynomial sample complexity bounds for revenue maximization
over all of the aforementioned ``simple'' auctions.
% show that
%measuring the complexity of an auction class using pseudo-dimension
%also applies naturally to multi-parameter domains, and propose a
%general technique for bounding the pseudo-dimension of a class of
%multi-parameter auctions.
In effect, our results imply that whenever it's possible to compute a
near-optimal simple auction with a known prior, it is also
possible to compute such an auction with an unknown prior (given a
polynomial number of samples).

One concrete example of a class of well-studied ``simple''
multi-parameter auctions comes from \citet{babaioffadditive}.
Consider a single bidder whose valuation is \emph{additive} over $k$
items: there is a vector $v \in \R^k$ such that the bidder's valuation
for a bundle $B\subseteq[k]$ is $ \sum_{j\in B} v_j$.  An \emph{item
  pricing} is defined by a vector $\p\in\R^k$, and offers the agent
any bundle $B$ for price $\sum_{j\in B} \p_j$. A \emph{grand bundle
  pricing} is defined by a single real number $q\in\R$ and offers the
bundle $B = [k]$ for the price $q$. When a single additive buyer's
valuation $v\sim \D_1 \times \cdots \times \D_k$ is drawn from a
product distribution, either the best item pricing or the best
grand-bundle pricing will earn $1/6$ of optimal revenue.
\citet{babaioffadditive} assume that the $D_j$'s are known a priori
and choose item and bundle prices as a function of the distributions.
Can we instead learn from samples the best auction from the class
consisting of all item and bundle prices?  The main result in
\citet{babaioffadditive} provides a bound on the representation error
of this class; our work provides the first sample complexity bound
(for this and many other classes).
% (the best
%prices will depend on $D_1,\ldots,D_k$).  Thus, 
%The number of parameters
%needed to specify an auction is the same as the number needed to
%specify a valuation in this setting.

\paragraph{Our Main Results} We present a general framework for
bounding sample complexity for ``simple'' combinatorial auctions, when
considering auctions as functions from valuations to
revenue. Formally, we study the following question, and provide a
technique for answering it in many interesting cases:

\begin{quote}
  Given a class of multiparameter auctions $C : \V^n \to \R$ (each
  auction maps any $n$-tuple of combinatorial valuations to the
  revenue achieved by the auction on those valuations), how large must
  $m$ be such that the empirical revenue maximizer in $C$ over $m$ samples
  drawn from $\D$ earns $OPT(C)-\epsilon$ expected revenue on fresh
  sample drawn from $\D$?
\end{quote}

Our main technical contributions are first to show a general way to
measure the sample complexity of single-buyer mechanisms, which are
interesting in their own right, and second to show a reduction in
bounding the sample complexity for the multi-buyer auctions to
bounding the sample complexity of single-buyer auctions. This
reduction and our general framework apply to all known multi-buyer
simple auctions from the literature.  We then instantiate this
framework and show that it is flexible enough to bound the sample
complexity of a large class of simple auction classes from the
literature.  In particular, we bound the pseudo-dimension of item
pricings, grand bundle pricings, and second-price item or bundle
auctions with reserves, covering the set of known simple auctions
which approximately optimize revenue. The following table summarizes
our results, as well as the known approximation guarantees these
auctions provide.
%\protect\footnote{
We
note that Theorem~\ref{thm:tighter-additive} is proven using a direct
argument rather than this framework.%}

\begin{footnotesize}
\noindent \begin{tabular}{ |p{2.4cm}||p{1.3cm}|p{2.4cm}|p{3.8cm}|p{3.8cm}|  }
 \hline
 \multicolumn{5}{|c|}{Summary of Simple Auction Properties} \\
 \hline
 Class & Valuations & PD anon, nonanon & Rev APX anonymous & Rev APX nonanonymous \\
 \hline
Grand bundle pricing & General & $O(1)$, $O(n)$\newline  Corollary~\ref{cor:bundle} &  &  \\\hline
Item Pricing & General & $\tilde{O}(k^2)$, $\tilde{O}(k^2n)$, Corollary~\ref{cor:items} & 3 ($1$ unit-demand bidder)\newline~\citep{chawla2007algorithmic} & 10.7 ($n$ unit-demand bidders)\newline~\citep{chawla2010multi} \\\hline
Item and Grand Bundle Pricing & General & $\tilde{O}(k^2)$, $\tilde{O}(k^2n)$ & 6 ($1$  additive bidder)\newline~\citep{babaioffadditive} & \\
 &  &  & 312 ($1$ subadditive bidder)\newline~\citep{rubensteinsubadditive}& \\
 & Additive & $\tilde{O}(k)$, $\tilde{O}(kn)$,\newline Theorem~\ref{thm:tighter-additive}&  &  \\\hline
Second-price item auctions with item reserves & Additive & $\tilde{O}(k^2), \tilde{O}(k^2n)$,\newline Corollary~\ref{cor:reserves} & & 48 ($n$ additive buyers),\newline~\citep{yao2015reduction}\\
 &  &  $\tilde{O}(k)$, $\tilde{O}(kn)$,\newline  Theorem~\ref{thm:tighter-additive} &  & \\
 \hline
\end{tabular}\end{footnotesize} 

These results imply that a polynomial-sized sample suffices to learn a
nearly-optimal auction from these classes of simple auctions. Thus,
when combined with results from the literature, it is possible to
learn auctions which earn a constant-factor of optimal revenue for a
single additive~\citep{babaioffadditive} or subadditive
bidder~\citep{rubensteinsubadditive}, $n$ unit-demand
bidders~\citep{chawla2010multi}, and $n$ additive
bidders~\citep{yao2015reduction}.\footnote{Our framework also applies
  to learning simple auctions with good welfare guarantees, as in
  \citet{feldman2015combinatorial}; all that changes is the
  real-valued function associated with an auction.  Welfare guarantees
  are simpler than revenue guarantees (since the objective function
  value depends on the allocation only) so we concentrate on the
  latter.}

For many classes of auctions, our sample complexity bounds do not rely
on any 
structural assumptions about buyers' valuations (only that their
utilities are quasilinear in money and that they will behave by
maximizing their own utility). We point to this flexibility as a key
feature of our techniques: for bidders with general valuation
functions, it can be quite complicated to reason about bidder's
behavior directly.  We also formally describe the allocations of these
auction classes as coming from \emph{sequential} allocation procedures
all drawn from the same class, and show any class which has
allocations which can be described this way also has a provably simple
class of allocation functions. This reduction may be of independent
interest for proving other classes of auctions have small sample
complexity.

\subsection{Related Work}\label{sec:rw}

There has been a recent surge in the design of single-parameter
revenue-maximizing auctions\footnote{A generalization of single-item
  auctions, where each buyer can be described by a single real number
  representing her value for being selected as a winner.}  from
samples~\citep{elkind2007,balcan2007CMUtechreport,balcan2008reducing,CR14,huang2014making,
  medina2014learning,RS15,morgenstern2015pseudo,devanurside2015}; we
focus here on the problem of designing auctions from samples for
multi-parameter settings. Optimal auctions for combinatorial settings
are substantially more complex than for single-parameter settings, even
before introducing questions of sample complexity.  Item pricings in
particular have been the subject of much study with respect to their
constant approximations when buyers' values for items are
independent,\footnote{\citet{dughmi2014sampling} show that when
  items' values are allowed to be correlated, for a single unit-demand
  bidder, the sample complexity required to compute a constant-factor
  approximation to the optimal auction is necessarily exponential (in
  $m$).}  for
welfare~\citep{kelso1982,feldman2015combinatorial}
and
revenue in the worst-case for for a
single~\citep{chawla2007algorithmic} and $n$ unit-demand
bidders~\citep{chawla2010multi}, a single~\citep{babaioffadditive}
additive buyers, and a single subadditive
buyer~\citep{rubensteinsubadditive}.\footnote{For more general
  valuation profiles and without item-wise independence, it is known
  that item pricings can also achieve super-constant revenue
  approximations, see~\citet{,balcan2008item,chakraborty2013dynamic}.}
For the additive and subadditive buyer results, the theorems state
that the better of the best item pricing and best grand bundle pricing
achieve a constant factor of optimal revenue. Combining the result
of~\citet{babaioffadditive} and a recent result
of~\citet{yao2015reduction}, it is possible to earn a constant factor
of optimal revenue for $n$ additive buyers using the better of the
best grand bundle pricing and an auction format related to item
pricing (the second-price item auction with item-specific reserve
prices).  All $n$-buyer results for revenue
rely on the use of \emph{nonanonymous}
item and grand bundle pricings.  
These results can be thought of as
bounding the representation error of using these classes of auctions
for revenue maximization; our work can be thought of as complementing
these results by bounding the classes' generalization error.

Item pricings are also sufficiently simple that the sample complexity
of choosing
welfare-optimal~\citep{feldman2015combinatorial,walras2016} and
revenue-optimal~\citep{balcan2008reducing} item pricings has been
explored.~\citet{balcan2008reducing} study the sample complexity of
anonymous item pricing for combinatorial auctions with unlimited
supply and employ one technique which bears some resemblance to our
framework. Fixing a sample of size $m$, they bound the number of
distinct allocation labelings $L$ of that sample by anonymous item
pricing using a geometric interpretation of anonymous pricings. Such
an argument seems difficult to extend to other classes of auctions (for
example, nonanonymous item pricings).  We ultimately suggest the use
of linear separability as a tool to bound $L$, an argument which
applies to many distinct classes of auctions with finite supply, and
doesn't rely on the particular geometry of anonymous item pricings.

We use the concept of linear
separability~\citep{daniely2014multiclass} to prove bounds on the
pseudo-dimension of many classes of auctions; this tool was also used
by~\citet{balcan2014revealed} in studying the sample complexity of
learning the valuation function of a single buyer when goods are
divisible and the valuation functions are either additive, Leontiff,
or Separable Piecewise-Linear Concave; our results apply to multiple
bidders, when items are indivisible, and most of them to arbitrary
valuation classes (including superadditive valuations).
~\citet{walras2016} also studied the linear separability and used it
to bound the pseudo-dimension of welfare maximization for item
pricings as well as the concentration of demand for any particular
good.

%%% Local Variables:
%%% mode: latex
%%% TeX-master: "paper"
%%% End:

\section{Preliminaries}\label{sec:prelims}
\paragraph{Bayesian Mechanism Design
  Preliminaries}
In this section, we provide the definitions and main results regarding
simple multi-parameter mechanism design necessary for proving our main
results.  We consider the problem of selling $k$ heterogeneous items
to $\nbuy$ bidders. Each bidder $i\in[\nbuy]$ can be described by a
\emph{combinatorial valuation function}
$v_i \in \V \subseteq (2^k\to \mathbb{R})$, and is assumed to be
\emph{quasilinear in money}, meaning that her utility for a bundle $B$
with price $p(B)$ is exactly $u_i(B, p) = v_i(B) - p(B)$. We will
assume all valuation functions are \emph{monotone}, $v(B) \leq v(B')$
for all $B \subseteq B'$.  An auction $\A$ is comprised of an
allocation rule $\A_1 : \V^n \to [\nbuy]^k$ and a payment rule
$\A_2 : \V^n \to \R^\nbuy$. We will only consider direct revelation
mechanisms, for which it is the best-response for any buyer to reveal
$v_i$ to any mechanism $\A$.  The valuation function $v_i$ is assumed
to be known to agent $i$ but not to the designer of the auction, who
must choose an auction $\A$ before observing $v_1, \ldots, v_n$.

We will assume that bidder $i$'s valuation is drawn independently from
some distribution $\D_i$ over valuation profiles. We assume the
support of the distribution $\D = \D_1 \times \ldots \times \D_\nbuy$
is in $[0, H]^\nbuy$.  We will refer to the \emph{revenue} of an
auction $\A$ on a particular instance $v = (v_1, \ldots, v_\nbuy)$ as
$\sum_{i}\A_2 (v)_i$, and the (expected) revenue of $\A$ as
$\rev(\A, \D) = \E_{v\sim \D}[\sum_{i}\A_2(v)_i]$. When a bidder's valuation $v_i$
can be represented as $v_{i1}, \ldots, v_{ik}$ such that
$v_i(S) = \sum_{j\in S}v_{ij}$, we say that $i$ is \emph{additive};
when $v_i$ can be represented as $k$ numbers $v_{i1}, \ldots, v_{ik}$
such that $v_i(S) = \max_{j\in S}v_{ij}$, we say that $i$ is
\emph{unit-demand}. If, for all $S, T\subseteq [k]$,
$v_i(S) + v_i(T) \geq v_i(S\cup T)$, we say $v_i$ is
\emph{subadditive}.

Several particular kinds of auctions are of particular use when
(approximately) optimizing for revenue in multi-parameter settings.
An auction is an (anonymous) \emph{item pricing} if it sets price
$p_j$ for each item $j\in [k]$, and offers buyers in some fixed order
any bundle $B$ of remaining items for price $\sum_{j\in B} p_j$; each
buyer then chooses the bundle maximizing her utility.  An auction is a
\emph{non-anonymous} item pricing if it sets price $p_{ij}$ for each
$j\in [k], i\in [\nbuy]$, and offers buyers in some fixed order any
bundle $B$ of remaining items; buyer $i$ will be offered $B$ for a
price of $\sum_{j\in B}p_{ij}$. An anonymous (or nonanonymous) grand
bundle pricing sets a single price $p$ ($p_i$) for the ``grand''
bundle $[k]$ of all items, and offers the grand bundle to buyers in
some fixed order until the grand bundle is sold. Throughout the paper,
we will assume this fixed order is fixed (namely, not a parameter of
the design space), and that it places bidder $1$ first in the
ordering, $2$ second, and so on.  When buyers are additive, we will
also consider the \emph{second-price} item auction, which sells each
item to the highest bidder for that item at the second-highest bid for
that item, and the second-price grand bundle auction which sells the
grand bundle to the highest bidder for it at the second-highest
bid. Finally, we will consider the second-price item (grand-bundle)
auction with both anonymous and non-anonymous item reserves, which
sell to the highest bidder for that item at the maximum of the
second-highest bid and the item's reserve (for that bidder), or to no
one if the highest bidder's bid is below her reserve. These auction
classes achieve constant-factor approximations for revenue in many
special cases: for one~\citep{chawla2007algorithmic} and
$n$~\citep{chawla2010multi} unit-demand bidders, for
one~\citep{babaioffadditive} and $n$~\citep{yao2015reduction} bidders,
and for one subadditive bidder~\citep{rubensteinsubadditive} (see
Section~\ref{sec:formal-known}, where we have included the formal
theorem statements for completeness).

\paragraph{Learning Theory Preliminaries}\label{sec:prelims-lt}

In this section, we provide definitions and useful tools for bounding
the sample complexity of learning a class of real-valued functions
$\F$.  We omit discussion of binary-labeled learning and the
definitions of uniform versus PAC learning for reasons of space (see
Section~\ref{sec:binary} for further details).

\paragraph{Real-Valued Labels}

Both PAC learnability and uniform learnability of binary-valued
functions can be well-characterized in terms of the class's VC
dimension. When learning real-valued functions (for example, to
guarantee convergence of the revenue of various auctions), we use a
real-valued analog to VC dimension (which will give a sufficient but
not necessary condition for uniform convergence). We will work with
the \emph{pseudo-dimension} \citep{pollard1984}, one standard
generalization.  Formally, let $c : \V \to [0,\maxval]$ be a
real-valued function over $\V$, and $\F$ be the class we are learning
over. Let $S$ be a sample drawn from $\D$, $|N|=\sam$, labeled
according to $c$.  Both the empirical and true error of a hypothesis
$\hat c$ are defined as before, though $|\hat c(v) - c(v)|$ can now
take on values in $[0, \maxval]$ rather than in $\{0,1\}$. Let
$(r_1, \ldots, r_\sam) \in [0,\maxval]^\sam$ be a set of
\emph{targets} for $N$. We say $(r_1, \ldots, r_\sam) $
\emph{witnesses} the shattering of $N$ by $\F$ if, for each
$T\subseteq N$, there exists some $c_T\in \F$ such that
$c_T(v_q) \geq r_q$ for all $v_q \in T$ and $c_T(v_q) < r_q$ for all
$v_q \notin T$. If there exists some $\vec r$ witnessing the
shattering of $N$, we say $N$ is {\em shatterable} by $\F$.  The {\em
  pseudo-dimension} of $\F$, denoted $\pd(\F)$, is the size of the
largest set $S$ which is shatterable by $\F$.  We will derive sample
complexity upper bounds from the following theorem, which connects the
sample complexity of uniform learning over a class of real-valued
functions to the pseudo-dimension of the class.

\begin{theorem}[E.g.~\citet{AB}]\label{thm:fat-sample}
  Suppose $\F$ is a class of real-valued functions with range in
  $[0,\maxval]$ and pseudo-dimension $\pd(\F)$. For every
  $\epsilon > 0, \delta \in [0,1]$, the sample complexity of
  $(\epsilon, \delta)$-uniformly learning the class $\F$ is
  \[
    n = O\left(
      \left(\frac{\maxval}{\epsilon}\right)^2\left(\pd(\F)\ln \frac{\maxval}{\epsilon}
        + \ln \frac{1}{\delta} \right)\right).
  \]
\end{theorem}
Moreover, a conceptually simple algorithm achieves the guarantee in
Theorem~\ref{thm:fat-sample}: simply output the function $c \in \F$ with the
smallest empirical error on the sample. These algorithms are called
\emph{empirical risk minimizers}.

\paragraph{Multi-Labeled Learning}

The main goal of our work is to bound the sample complexity of revenue
maximization for multi-parameter classes of auctions (via bounding
these classes' pseudo-dimension); our proofs first bound the number of
labelings of \emph{purchased bundles} which these auctions can induce
on a sample of size $m$. Then, we argue about the behavior of the
revenue of all auctions which agree on the purchased bundles for every
sample to bound the pseudo-dimension. Since bundles are neither binary
nor real-valued, we now briefly mention several tools which we use for
learning in the so-called \emph{multi-label} setting.

The first of these tools is that of \emph{compression schemes} for a
class of functions. 

\begin{definition}
A \emph{compression scheme} for $\F : \V \to \Y$,  of size $d$ consists of
\begin{itemize}
  \item a \emph{compression} function \[\com : (\V \times \Y)^{\sam} \to (\V \times
      \Y)^{d},\] where $\com(N) \subseteq N$ and  $d \leq \sam$; and
  \item a \emph{decompression} function
    \[\decom : (\V\times \Y)^d \to \F.\]
\end{itemize}
For any $f\in \F$ and any sample
$(v_1, f(v_1)), \ldots, (v_\sam, f(v_\sam))$, the functions satisfy
\[
  \decom \circ \com ((v_1, f(v_1)), \ldots, (v_\sam, f(v_\sam))) = f'\in \F
\]
where $f'(v_q) = f(v_q)$ for each $q\in [\sam]$.
\end{definition}

Intuitively, a compression function selects a subset of $d$ ``most
relevant'' points from a sample, and based on these
points, the decompression scheme selects a hypothesis.  When such a
scheme exists, the learning algorithm $\decom \circ \com$ is an
empirical risk minimizer. Furthermore, this compression-based learning
algorithm has sample complexity bounded by a function of
$d$, which plays a role analogous to VC dimension in the sample
complexity guarantees.

\begin{theorem}[\citet{littlestone1986compression}]
\label{thm:compression-sample}
Suppose $\F$ has a compression scheme of size $d$.
Then, the $\PAC$ complexity of $\F$ is at most
$\sam = O\left(\frac{d \ln\frac{1}{\epsilon} +
   \ln \frac{1}{\delta}}{\epsilon}\right)$.
\end{theorem}

While compression schemes imply useful sample complexity bounds, it
can be hard to show that a particular hypothesis class admits a
compression scheme. One general technique is to show that the class is
linearly separable in a higher-dimensional space.

\begin{definition}
  A class $\F$ is \emph{$d$-dimensionally linearly separable} if there
  exists a function $\psi : \V \times \Y \to \R^d$ and for any
  $f\in \F$, there exists some $w^f\in \R^d$ with
  $f(v) \in \argmax_{y}\langle w^f, \psi(v, y)\rangle$ and
  $|\argmax_{y}\langle w^f, \psi(v, y)\rangle| = 1$.
\end{definition}
It is known that a $d$-dimensional linearly separable class admits a compression
scheme of size $d$.
\begin{theorem}[Theorem 5 of~\citet{daniely2014multiclass}]\label{thm:linsep-compression}
  Suppose $\F$ is $d$-dimensionally linearly separable. Then, there
  exists a compression scheme for $\F$ of size $d$.
\end{theorem}
If a class is linearly separable, this greatly restricts the number of
labelings it can induce on a sample of size $m$, a trick used
in~\citet{walras2016} and also in the next section of this paper.

We also briefly mention that if a class $\F$ is linearly separable,
post-processing the class with a fixed function also yields a linearly
separable class over the resulting label space.

\begin{observation}\label{obs:postprocessing}
  Suppose $\F$ is $d$-dimensionally linearly separable over $Q$. Fix
  some $q : Q \to Q'$. Then, the set
  $q \circ \F = \{q \circ f | f\in\F\}$ is $d$-dimensionally linearly
  separable over $Q'$.
\end{observation}
% 
%Finally, as proved in~\cite{walras2016}, the mere existence of a
%compression scheme for a class $\F$ upper-bounds the number of
%labelings the class can induce on a sample of size $\sam$.
%\begin{lemma}[\citet{walras2016}]\label{lem:labelings-comp}
%  Suppose $\F$ has a compression scheme of size $d$. Then, $\F$ can
%  induce at most ${\sam \choose d} |\Y |^{d}$
%  distinct labelings for any sample $S$ of size $\sam$. 
%%Thus, if $S$ is
% % shatterable, $\sam \leq d (\ln(\sam \cdot |\Y|))$.
%\end{lemma}
%\jmcomment{If one can come up with a better upper bound on the number
%  of labelings, then this would also work... such as for
%  unit-demand/additive bidders. This definitely loses a factor of $nk$
%  for them, though we don't know the answer for subadditive.}
%
With these tools in hand, our roadmap is as follows: for a class of
auctions, we first prove that the class (which labels valuations by
utility-maximizing bundles purchased) is linearly separable, which
then implies an upper-bound on how many distinct bundle labelings one
can have for a fixed sample. Then, we argue about the pseudo-dimension
of the class (which labels a valuation by the revenue achieved when
that agent buys her utility-maximizing bundle) by considering only those
auctions which all have the same bundle labeling of $\sam$ samples and
arguing about the behavior of revenue of those auctions.

\newcommand{\pset}{\mathcal{P}}
 \renewcommand{\H}{\mathcal{H}}
\newcommand{\vs}[1]{v_{#1\Sigma}}
\renewcommand{\B}{{\bf B}}

\section{A Framework for Bounding Pseudo-dimension 
  Via Intermediate Discrete Labels}\label{sec:discrete}

We now propose a new framework for bounding the pseudo-dimension of
many well-structured classes of real-valued functions. Suppose $\F$ is
some set of real-valued functions whose pseudo-dimension we wish to
bound. Suppose that, for each $f\in \F$, $f$ can be ``factored'' into
a pair of functions $(f_1, f_2)$ such that $f_2(f_1(x),x) = f(x)$ for
any $x$. There are always ``trivial'' factorings, where the function
$f_2 = f$ or $f_1(x) = x$, but the interesting case arises when both
$f_1(x)$ and $f_2$ (fixing $f_1(x)$) depend in a very limited way upon
$x$. In particular, if the set of functions $\{f_1\}$ are very
structured, and fixing $f_1(x)$ the set of functions $\{f_2 \}$ only
depend upon $x$ in some very mild way, this will imply that $\F$
itself has small pseudo-dimension. Intuitively, this will allow us to
``bucket'' functions by their values according to $f_1$ on some
sample, and bound the pseudo-dimension of each of those buckets
separately.

Our particular technique for showing such a property is first to show
that the set of functions $\{f_1\}$ are \emph{linearly separable} in
$a$ dimensions, then to fix some sample $S$ of size $m$ and some
$f_1$, and to upper-bound by $b$ the pseudo-dimension of the set of
functions $f_2$ whose associated $f'_1$ agrees with the labeling of
$f_1$ on $S$.  The following definition captures precisely what we
mean when we say that the function class $\F$ \emph{factors} into
these two other classes of functions. If $f_1(x)$ reveals too much
about $x$, it will be difficult to prove linear separability;
similarly, if $f_2$ depends too heavily on $x$, it will be difficult
to prove a bucket has small pseudo-dimension.

\begin{definition}[$(a, b)$-factorable class]\label{def:structured}
\normalfont
  Consider some $\F = \{f : \X \to \R\}$. Suppose, for each $f\in \F$,
  there exists $(f_1, f_2), f_1 : \X \to \Y, f_2 : \Y \times \X \to \R$ such
  that $f_2(f_1(x), x) = f(x)$ for every $x\in\X$. Let
\[\F_1 = \{f_1 : (f_1, f_2) \textrm{ is a decomposition of some }f\in\F\}\]
and 
\[\F_2 = \{f_2 : (f_1, f_2) \textrm{ is a decomposition of some }f\in\F\}.\]
The set $\F$ {\em $(a,b)$-factors over $Q$} if:
\begin{itemize}

\item [(1)]
$\F_1$ is $a$-dimensionally linearly separable over
$Q\subseteq \Y$. 

\item [(2)]
For every $f_1\in \F_1$ and sample $S\subset\X$
of size $m$, the set
\[\F_{2|f_1(S)} = \{f'_2 : \X \to \R, f'_2(x) = f_2(f'_1(x),x)  | f_1(S) = f'_1(S) \textrm{ and }(f'_1,
f_2) \textrm{ is a decomposition of some }f\in\F\}\]
has pseudo-dimension at most $b$. 
%on such sets $S$, we will say $\F$
%can $(a,b)$-factor over $Q$.

\end{itemize}
\end{definition}

We now give an example of a simple class which satisfies this
definition. One could easily bound the pseudo-dimension of this
example class using a direct shattering argument, but it will be
instructive to work through our definition of $(a,b)$-separability.
\begin{example}
  \normalfont Fix some set $G = \{g_1, \ldots, g_k\} \subset
  \R^k$.
  Suppose $\F = \{f : f(x) = \max_{g\in G_f \subseteq G} g \cdot x\}$
  is the set of all functions which take the maximum of at most $k$
  common linear functions in a fixed set $G$. We will show that $\F$
  $(kd, \tilde{O}(kd))$-factors over $[k]$, where each $j\in [k]$ will
  represent \emph{which} of the $k$ linear functions is maximizing for
  a particular input. That is, for some $f, G_f\subseteq G$, let
  $f_1(x) = \argmax_{t : g_t\in G_f} g_t \cdot x$ and
  $f_2(t,x)= g_t \cdot x$. Thus, we have a valid factoring:
\[f_2(f_1(x),x) = f_2(\argmax_{t : g_t\in G_f} g_t\cdot x, x) =
  g_{\argmax_{t: g_t\in G_f} g_t \cdot x}\cdot x = \max_{g_t\in G_f}
  g_t \cdot x = f(x).\]
  It remains to show that $\F_1$ is $d$-dimensionally linearly
  separable and to bound the pseudo-dimension of $\F_{2|f_1}$. We
  start with the former. Let
  $\Psi(x, t)_{t'j} = \I[t' = t] \cdot x_{j}$ for
  $t'\in [k], j\in[d]$. Then, let
  $w^f_{tj} = \I[g_t\in G_f] \cdot g_{tj}$. The dot product will then
  be
\[\Psi(x, t) \cdot w^f =\sum_{t'}  \I[t' = t] \cdot \I[g_{t'}\in G_f] g_{t'} \cdot x\]
which will be maximized when
$t = \argmax_{t' : g_{t'} \in G_f} g_{t'} \cdot x$, or when
$t = f_1(x)$. So, $\F_1$ is linearly separable in $kd$ dimensions over
$[k]$.

Now, fix $f_1\in \F_1$; we will show the pseudo-dimension of
$\F_{2| f_1}$ is at most $\tilde{O}(kd)$. For any fixed sample
$S = (x^1, \ldots, x^m)$, $f_1(x^t)$ is fixed for all $t\in[m]$,
implying that the input to all $f_2\in \F_{2|f_1}$, $(f_1(x^t), x^t)$,
is fixed.  Finally, by definition of $f'_2$,
\[f'_2(x^t) = f_2(f_1(x^t), x^t) = g_{f_1(x^t)} \cdot x^t.\]
Thus, for each $j\in [k]$, the subset $S_j\subseteq S$ for which
$f_1(x^t) = j$ for all $x^t\in S_j$, $f'_2$ is just a linear function
in $d$ dimensions of $x^t$ with coefficients $g_j$, Thus, since linear
functions in $d$ dimensions have pseudo-dimension at most $d+1$, there
are at most $m^{d+1}$ labelings which can be induced on $S_j$, and at
most $m^{k(d+1)}$ labelings of all of $S$. This implies
$\pd(\F_{2|f_1})$ is at most $\tilde{O}(kd)$.
\end{example}

We now present the main theorem about the pseudo-dimension of classes
which are $(a,b)$-factorable. The proof of this theorem first exploits
the fact that linearly separable classes have a ``small'' number of
possible outputs for a sample of size $m$. Then, fixing the output of
the linearly separable function, the second set of functions'
pseudo-dimension is small. The proof of the theorem is relegated to
the appendix due to space considerations. 

\begin{theorem}\label{thm:linsep}
Suppose $\F$ is $(a,b)$-factorable over $Q$. Then, 
\[\pd(\F) = O\left(\max\left((a+b)\ln(a+b), a\ln|Q|\right)\right) .\]
\end{theorem}

Intuitively, when $\F_1$ is linearly separable in $a$ dimensions, it
can induce at most $m^a |Q|^a$ many labelings of $m$ samples, and
fixing such a sample and its labeling, because $\F_2$ has
pseudo-dimension at most $b$, it can induce at most $m^b$ labelings of
$m$ samples with respect to their thresholds.

While the range of $\F_1$ might be all of $Q$, it will regularly be
helpful to only need to prove linear separability of $\F_1$ only over
``realizable'' labels for particular inputs. If $\F_1$ has the
property that for every input $x$, every $f_1\in\F_1$ labels $x$ with
one of a smaller set of labels $Q_x\subsetneq Q$, then it suffices to
prove linear separability \emph{for $x$} over $Q_x$. The following
remark makes this claim formal; its proof can be found in
Section~\ref{sec:omitted}. So, we will be able to focus on proving
linear separability of $\F_1$ over a label space which depends upon
the inputs $x$. This will be particularly useful when describing
auctions in the next section, whose allocations are in certain cases
highly restricted by their inputs.

\begin{remark}\label{rem:ind-labels}
  Suppose for each $x\in \X$, there exists some $Q_x \subseteq Q$ such
  that $f_1(x) \in Q_x\subseteq Q$ for all $f_1\in \F_1$, and that for
  for each $x$, $\F_1$ is linearly separable in $a$ dimensions for
  that $x$ over $Q_x$. Assume there is a subset of dimension
  $T^+ \subseteq [a]$ for which $w^f_{t\in T^+} \geq 0$ and
  $\sum_{t\in T^+} w^f_{t} > 0$ for all $f$.  Suppose that for all
  $x\in X, f\in \F_1$, $\max_{y\in Q_x} \Psi(x, y) \cdot w^f \geq 0$.
  Then, $\F_1$ is linearly separable over $Q$ in $a$ dimensions as
  well.
\end{remark}

%%% Local Variables:
%%% mode: latex
%%% TeX-master: "paper"
%%% End:

\section{Consequences for Learning Simple Auctions}\label{sec:auctions}

We now present applications of the framework provided by
Theorem~\ref{thm:linsep} to prove bounds on the pseudo-dimension for
many classes of ``simple'' mechanisms. The implication is that these
classes, which have been shown in many special cases to have small
\emph{representation error} also have have small \emph{generalization
  error} when auctions are chosen after observing a polynomially sized
sample. We now describe how one can translate a class of auctions into
a class of functions which has an obvious and useful factorization. An
auction $\A : \V^n \to [n]^k \times [0,H]^n$ has two components, its
\emph{allocation function} $\A_1 : \V^n \to [n]^k$ and its
\emph{revenue function} $\A_2 : \V^n \to [0,H]^n$.  We will abuse
notation and refer to $\A_2(\v) = \sum_{i} A(\v)_{2i}$ as the revenue
function for an auction.  Our goal is to bound the sample complexity
of picking some high-revenue function from a class. All omitted proofs
are found in Section~\ref{sec:omitted}.  For the remainder of this
section we use $\F$ to represent a class of auctions, $f\in\F$ to
represent a particular auction, and $\F_1, \F_2$ to be the
corresponding allocation functions and revenue functions which result
from this decomposition. When $\F_1$ is linearly separable, this
implies there can only be so many distinct allocations possible for a
fixed set of valuation profiles $S$, and when $\F_2$ (fixing some
allocation for $S$) has small pseudo-dimension, the class of auctions
itself has small pseudo-dimension.

This ``trivial'' decomposition of an auction's revenue function
describes its revenue function as a function of both the allocation
chosen by $f_1\in\F_1$ for $\v$ and the valuation profile $\v$. Since
$A_2$ is a function only of $\v$, there is clearly enough information
in $(f_1(\v), \v)$ to compute $\A_2(\v)$ (one can simply ignore
$f_1(\v)$ and output $f_2(f_1(\v), \v) =\A_2(\v)$). The reason we
consider this decomposition is that fixing an allocation, revenue
functions of simple auctions are generally very simple to describe as
a function of the input valuation profile $\v$.  If one fixes the
allocation choice for a sample $S$ of $m$ valuations, many auctions'
classes of revenue functions are either constant functions on $S$
which do not depend upon $\v$ at all (for example, a posted price
auction for a single item offered to a single bidder earns its posted
price if the item sells and $0$ when the item doesn't sell, both of
which are constants when the allocation is fixed) or depends only in a
very mild way (for example, a second-price single-item auction with a
reserve earns the maximum of its reserve and the second-highest bid
when the item sells and 0 when it doesn't).

Most of the ``simple'' auctions with multiple buyers and items
that have been considered are {\em sequential
  auctions} which interact
with buyers ``one at a time'': first, bidder $1$ is offered a menu of
several possible allocations at different prices, she chooses some
bundle, then bidder $2$ is offered one of several allocations of the
remaining items, and so on. We assume for the remainder of the paper
that there are no ties, (that is, there are no menus or bidders for
which $|\argmax_{B}u(B)| > 1$.\footnote{ We elide further discussion
  on this technical point, though we note it is possible to encode a
  tie-breaking rule over utility-maximizing bundles in a way which is
  linearly separable (see~\citet{walras2016} for more details).}
These auctions are simple enough that they can actually be run in
practice, and yet expressive enough that in many cases can earn
constant fractions of the optimal revenue.  

We next work toward a general reduction, from bounding the sample
complexity of sequential auctions (with multiple buyers) to that of
single-buyer problems.
The following definition
captures two particularly common forms of these auctions. The first
definition captures the setting where the function selecting the menu
to bidder $i$ may depend upon $i$'s identity; the second refers to
when the menu is \emph{anonymous}: what may be offered to bidder $i$
can be different than what is offered to bidder $i'$, but only due to
the differences in bids $v_i, v_{i'}$ and the remaining available
items $X_i(v), X_{i'}(v)$.

For example, consider a single item for sale. Suppose $n$ buyers are
approached in some fixed order and bidder $i$ is offered the item at
price $p_i$ if no earlier buyer has purchased the item. If
$p_i = p_{i'}$ for all $i, i' \in [n]$, then the auction applied to
each buyer is the same, and we say this auction applies an $n$-wise
repeated allocation associated with a single posted price. If
$p_i \neq p_{i'}$ for some $i, i'\in[n]$, then the allocation function
applied to each bidder is an allocation rule associated with some
single posted price, though the particular posted price and therefore
the allocation function varies from bidder to bidder; this auction's
allocation is therefore an $n$-wise sequential allocation drawn from
the class of posted prices.

For a slightly more complex example, consider a set of $k$
heterogeneous items $[k]$ for sale to $n$ bidders. Consider an
auction which sets a price $p_{ij}$ for each item $j\in[k]$ and each
bidder $i\in[n]$, and serves bidders in some fixed order. Bidder $i$ is
offered any bundle $B$ for which no item has been selected by some
previous bidder at price $p_i(B) = \sum_{j\in B}p_{ij}$.  This
allocation is reached by applying a posted item pricing allocation to
each buyer in turn, so these allocations are $n$-wise sequential
allocations drawn from posted item pricing allocations. If
$p_{ij} = p_{i'j}$ for all $j\in [k]$ and all $i, i'\in[n]$, then the
same allocation rule is being applied to all bidders, and the overall
allocation is therefore an $n$-wise anonymous sequential allocation
rule.

\begin{definition}[$n$-fold anonymous and nonanonymous sequential
  allocations] \label{def:sequence} Let $\H$ be some class with
  $h: \V \times \{0,1\}^k \to Q$ for all $h\in \H$ and some
  $Q\subseteq \{0,1\}^k$.  For some $n$ functions
  $h_1, \ldots, h_n \in \H$ and every $\v\in\V^n$, inductively define
  $X_1(\v) = [k], X_i(\v) = X_{i-1}(\v)\setminus h_{i-1}(\v_{i-1},
  X_{i-1}(\v))$.
  Then, define the $n$-wise product function
  $\prod_{(h_1, \ldots, h_n)}$ to be
  \[\prod_{(h_1, \ldots, h_n)} (\v) = (h_1(\v_1, X_1(\v)), h_2(\v_2, X_2(\v)),
  \ldots h_n(\v_n, X_n(\v))).\]
  Then, we call any such function an $n$-wise \emph{nonanonymous
    sequential allocation drawn from $\H$}, or $n$-wise sequential
  allocation drawn from $\H$ for short.  If
  $h_1 = h_2 = \ldots = h_n$, we call $\prod_{h_1, \ldots, h_n}$ an
  $n$-wise anonymous sequential allocation drawn from $\H$.
\end{definition}

The sets $X_1, \ldots, X_n$ correspond to the sets of remaining
available items for each bidder after the previous bidders have
purchased their bundles according to their allocation functions: what
is remaining for bidder $i$ is whatever bidder $i-1$ had available
less whatever was allocated to bidder $i-1$. The two previous examples
fit into this scenario perfectly. The per-bidder allocation functions
are fixed up-front: the allocation rules brought about by (anonymous)
a price for a single item or (anonymous) prices for each item. In some
fixed order, the bidders are allocated according to their allocation
rule run on their valuation and the remaining items, and whatever
items they didn't purchase are available for the next bidder and her
allocation rule. When the prices don't depend on the index $i$, the
allocation function for each bidder is the same, so those cases
correspond to $n$-wise anonymous sequential allocation rules.

In the event that some class of auctions' allocation functions $\F_1$
are made up of $n$-wise sequential allocations from a class $\H$ which
is linearly separable, the linear separability is imparted upon
$\F_1$. This intuition is made formal by the following theorem.

\begin{theorem}\label{thm:sequential}
  Suppose $\F$ is a class of auctions, and let
  $\F_1 : \V^n \to Q\subseteq [n]^k$ be their (feasible) allocation
  function. Suppose $\F_1$ is a set of $n$-wise sequential allocations
  from some $\H$ which is $a$-dimensionally linearly separable, whose
  dot products are upper-bounded by $\maxval$. Then $\F_1$ is
  $an$-dimensionally linearly separable.  Similarly, if $\F_1$ is a
  set of $n$-wise anonymous sequential allocations drawn from $\H$
  which is $a$-dimensionally linearly separable, then $\F_1$ is also
  $a$-dimensionally linearly separable.
\end{theorem}
Roughly speaking, this proof takes the maps guaranteed by linear
separability of $\H$ and concatenates them $n$ times, ``blowing up''
the relative importance of the earlier bidders with large
coefficients.

We now present the three main corollaries of Theorems~\ref{thm:linsep}
and~\ref{thm:sequential} which bound the pseudo-dimension of several
auction classes of interest to the mechanism design community. In
particular, we focus on ``grand bundle'' pricings
(Corollary~\ref{cor:bundle}), where each bidder in turn is offered the
entire set of items $[k]$ at some price, ``item pricings''
(Corollary~\ref{cor:items}), where each bidder in turn is offered all
remaining items and each item $j$ has some price for purchasing it,
and ``second-price item auctions with reserves''
(Corollary~\ref{cor:reserves}), where each bidder submits a bid for
each item $j$, and item $j$ is sold to the highest bidder for $j$ at
the larger of the item's reserve price and the second-highest bid for
that item. Each of these auctions have two versions: the anonymous
version, where the relevant design parameters are the same for all
bidders, and the non-anonymous version, where those parameters can be
bidder-specific. As one would suspect, anonymous pricings have fewer
degrees of freedom, and have lower pseudo-dimension. More formally,
the allocations which come from anonymous pricings can be formulated
as $n$-wise repeated allocations, while we formulate non-anonymous
pricings' allocations as $n$-wise sequential allocations (which, by
Theorem~\ref{thm:sequential} loses a factor of $n$ in the upper bound
on these classes' pseudo-dimensions). In each case, $\F_1$ will
represent allocation functions: $f_1\in\F_1$ corresponds to the
allocation function which the auction will implement for quasilinear
bidders. For every class $\F$, we define for every auction $f\in\F$
the function $f_2$ to be the \emph{revenue} function, which as a
function of an allocation and the valuation profile outputs the
revenue for that auction with that allocation for that valuation
profile. The decomposition of $\F$ into $\F_1, \F_2$ is trivial; the
work comes in showing that $\F_1$ is linearly separable and
$\F_{2|f_1}$ has small pseudo-dimensions.

Our first two results use the framework to that grand bundle pricings
and item pricings have small pseudo-dimension. The second case
requires a more delicate treatment of the valuation profiles (buyers
are now choosing arbitrary subsets of items, and will choose
utility-maximizing bundles based on the per-item prices). It also
requires us to consider a larger set of intermediate labels (the set
of all possible allocations grows to $[n]^k$ from $[n]$).
\begin{corollary}\label{cor:bundle}
Let $\F$ be the class of anonymous grand bundle pricings. Then, 
\[\pd(\F) = O(1) .\]
If $\F$ is the class of non-anonymous grand bundle pricings, then
\[\pd(\F) = O(n\log n).\]
\end{corollary}

\begin{corollary}\label{cor:items}
  Let $\F$ be the class of anonymous item pricings. Then, 
\[\pd(\F) = O\left(k^2\right).\]
If $\F$ is the class of nonanonymous item pricings, then
\[\pd(\F) = O\left(nk^2\ln(n)\right).\]
\end{corollary}

Finally, we present our final application of this technique, and bound
the pseudo-dimension of the class of second-price item auctions with
(non-anonymous)
item reserves. A result of~\citet{yao2015reduction} implies this class
has small representation error for additive buyers;
Corollary~\ref{cor:reserves} shows it also has small generalization
error. We briefly note that we have a slightly tighter bound on this
class's pseudo-dimension, using a stylized argument found in
Appendix~\ref{sec:additive}.

\begin{corollary}\label{cor:reserves}
  Suppose $\V$ is some set of additive valuations.  Let $\F$ be the
  class of second-price item auctions with anonymous reserves. Then,
\[\pd(\F) = O\left(k^2\right).\]
If $\F$ is the class of second-price item auctions with nonanonymous
reserves, then
\[\pd(\F) = O\left(nk^2\ln(n)\right).\]
\end{corollary}

\begin{proof}
  We will show that for both classes, $\F_1$ is linearly separable (in
  $k$ and $nk$ dimensions, respectively). We do this by showing that
  the anonymous class's' allocations can be described as $n$-wise
  anonymous sequential allocations and the nonanonymous class's
  allocations as $n$-wise nonanonymous sequential allocations from
  some class of single-buyer allocation rules $\H$ which is
  $k$-dimensionally linearly separable. The natural candidate for $\H$
  is the set of allocations defined by item pricings, which we showed
  in the proof of Corollary~\ref{cor:items} is linearly separable in
  $k$ dimensions.

  The fact that $\F_1$'s feasible allocations can be described as
  $n$-wise anonymous and nonanonymous sequential allocations over
  (single buyer) item pricings is not immediately obvious: a player's
  price for an item $j$ is not just the item's reserve price, but the
  maximum of that reserve and the second-highest bid for $j$. Thus, a
  buyer maximizing her quasilinear utility with respect to her
  reserves would not necessarily purchase the same bundle as when
  facing item prices which are the larger of her reserve and the
  second-highest bid for the item, even if she is the highest bidder
  for each item. However, since the valuations are \emph{additive},
  $i^*_j$ the highest bidder for $j$ \emph{will} maximize her utility
  by purchasing item $j$ if $v_{i^*_j}({j}) \geq p_j$, since
  $v_{i^*_j}({j}) \geq \max_{i' \neq i^*_j} v_{i'}({j})$ implies
  $v_{i^*_j}({j}) \geq \max\left(p_j, \max_{i' \neq i^*_j} v_{i'}({j})\right)$,
  and an \emph{additive} bidder will buy any item $j$ for which her
  value for that item is (weakly) higher than the price for that
  item. Thus, the utility-maximizing bundle for some $i$ with respect
  to her item prices \emph{over the set of bundles for which she only
    wins items for which she is the highest bidder} will also be
  utility-maximizing for an additive bidder needing to pay
  $ \max(p_j, \max_{i' \neq i^*_j} v_{i'}({j})$ for each $j$.
  Remark~\ref{rem:ind-labels} implies that we need only show linear
  separability over $Q_\v$ for each $\v\in \V^n$, where
  $Q_\v = \{f_1(\v)| f_i\in\F_1, \v\in\V^n\}$ is the range that the
  allocation rules might have for a particular $\v$. Thus, since this
  class only sells $j$ to the highest bidder for $j$, we need only
  show linear separability over allocations for good $j$ is either
  unallocated or sold to the highest bidder for $j$.  Thus, the
  allocation for each bidder $i$ will be correctly predicted by the
  item pricing linear separator (over the label space which only has
  highest bidders winning items).

  Thus, by Theorem~\ref{thm:sequential}, the second-price item
  auctions with anonymous item reserves is linearly separable in $k+1$
  dimensions, and with nonanonymous reserves in $n(k+1)$
  dimensions. We will now show that for both classes, for any $f\in\F$
  and corresponding $f_1\in\F_1$, the class $\F_{2|f_1}$ has
  pseudo-dimension $\tilde{O}(k)$ and $\tilde{O}(nk)$, respectively.

  First, fix $\F$ to be the set of the second-price item auctions with
  anonymous item reserves and pick some $f_1\in \F_1$. Then, for each
  $f'_2\in \F_{2|f_1}$ and each $\v^t\in S$, we have
  \[f'_2(\v^t) = f_2(f_1(\v^t), \v^t) = \sum_{j} \max(p^f_j,
  \max_{i'\neq i^*_j} \v^t_{i'}(\{j\})) \cdot f_1(\v^t)_j.\]
  For each item $j$, suppose the relative ordering of $p^f_j$ and
  $\max_{i'\neq i^*_j} \v^t_{i'}(\{j\})$ were fixed. Then, $\f'_2(\v^t)$
  is just a \emph{linear} function in $k$ dimensions of $\v^t$ and
  $p^f$, which have pseudo-dimension at most $k+1$, and therefore can
  induce at most $m^{k+1}$ labelings with respect to
  $(r^1, \ldots, r^m)$.  There are ${m+1}$ possible relative orderings
  of these parameters, or $(m+1)^k$ for all items
  simultaneously. Thus, in total, there can be at most
  $m^{k+1} \cdot (m+1)^k$ labelings of $S$ with respect to
  $(r^1, \ldots, r^m)$ by $\F_{2|f_1}$, so
  $\pd(\F_{2|f_1}) = O(k\ln(k))$.
 
  The proof that non-anonymous item reserves has pseudo-dimension
  $O(nk \ln(nk))$ is analogous, with a few small exceptions. First, we
  consider those $p^f\in \R^{nk}$ with a fixed ordering of (for all
  $j\in [k]$) the parameters $\{\p^f_{ij} | i \in [n] \}$ and the set
  $\{\v^t_{i'_j}(\{j\})| t\in [m]\}$; there are therefore
  $mn \choose n$ of these relative orderings for a fixed item, or
  $O((mn)^{nk})$ over all items. Fixing this ordering, the revenue on
  each sample is again a linear function (in $nk$ dimensions) of
  $\v^t, \p^f$ for each $\v^t$. Thus, $\F_{2|f_1}$ can induce at most
  $m^{nk+1}\cdot (mn)^{nk}$ many labelings of $S$ w.r.t.
  $(r^1, \ldots, r^m)$, implying $\pd(\F_{2|f_1}) = O( nk \ln(nk))$.

  Then, applying Theorem~\ref{thm:linsep} to the two classes (which
  are $(k,k\log k)$-factorable over
  $Q \subseteq \{0,1\}^k, |Q| \leq 2^k$ and
  $(nk, nk\ln(nk))$-separable over $Q \in [n]^k$) implies the
  pseudo-dimensions are at most $O(k^2)$ and $O(nk^2\ln(n))$,
  respectively.
\end{proof}

%%% Local Variables:
%%% mode: latex
%%% TeX-master: "paper"
%%% End:

{\footnotesize{\bibliography{sources}}}

\begin{thebibliography}{31}
\providecommand{\natexlab}[1]{#1}
\providecommand{\url}[1]{\texttt{#1}}
\expandafter\ifx\csname urlstyle\endcsname\relax
  \providecommand{\doi}[1]{doi: #1}\else
  \providecommand{\doi}{doi: \begingroup \urlstyle{rm}\Url}\fi

\bibitem[Anthony and Bartlett(1999)]{AB}
Martin Anthony and Peter~L. Bartlett.
\newblock \emph{Neural Network Learning: Theoretical Foundations}.
\newblock Cambridge University Press, NY, NY, USA, 1999.

\bibitem[Babaioff et~al.(2014)Babaioff, Immorlica, Lucier, and
  Weinberg]{babaioffadditive}
Moshe Babaioff, Nicole Immorlica, Brendan Lucier, and S.~Matthew Weinberg.
\newblock A simple and approximately optimal mechanism for an additive buyer.
\newblock In \emph{Symposium on Foundations of Computer Science (FOCS 2014)}.
  IEEE -- Institute of Electrical and Electronics Engineers, October 2014.

\bibitem[Balcan et~al.(2007)Balcan, Blum, and Mansour]{balcan2007CMUtechreport}
Maria-Florina Balcan, Avrim Blum, and Yishay Mansour.
\newblock Single price mechanisms for revenue maximization in unlimited supply
  combinatorial auctions.
\newblock Technical report, Carnegie Mellon University, 2007.

\bibitem[Balcan et~al.(2008{\natexlab{a}})Balcan, Blum, Hartline, and
  Mansour]{balcan2008reducing}
Maria-Florina Balcan, Avrim Blum, Jason~D Hartline, and Yishay Mansour.
\newblock Reducing mechanism design to algorithm design via machine learning.
\newblock \emph{Jour. of Comp. and System Sciences}, 74\penalty0 (8):\penalty0
  1245--1270, 2008{\natexlab{a}}.

\bibitem[Balcan et~al.(2008{\natexlab{b}})Balcan, Blum, and
  Mansour]{balcan2008item}
Maria-Florina Balcan, Avrim Blum, and Yishay Mansour.
\newblock Item pricing for revenue maximization.
\newblock In \emph{Proceedings of the 9th ACM conference on Electronic
  commerce}, pages 50--59. ACM, 2008{\natexlab{b}}.

\bibitem[Balcan et~al.(2014)Balcan, Daniely, Mehta, Urner, and
  Vazirani]{balcan2014revealed}
Maria-Florina Balcan, Amit Daniely, Ruta Mehta, Ruth Urner, and Vijay~V
  Vazirani.
\newblock Learning economic parameters from revealed preferences.
\newblock In \emph{Web and Internet Economics}, pages 338--353. Springer, 2014.

\bibitem[Chakraborty et~al.(2013)Chakraborty, Huang, and
  Khanna]{chakraborty2013dynamic}
Tanmoy Chakraborty, Zhiyi Huang, and Sanjeev Khanna.
\newblock Dynamic and nonuniform pricing strategies for revenue maximization.
\newblock \emph{SIAM Journal on Computing}, 42\penalty0 (6):\penalty0
  2424--2451, 2013.

\bibitem[Chawla et~al.(2007)Chawla, Hartline, and
  Kleinberg]{chawla2007algorithmic}
Shuchi Chawla, Jason Hartline, and Robert Kleinberg.
\newblock Algorithmic pricing via virtual valuations.
\newblock In \emph{Proceedings of the 8th ACM Conf. on Electronic Commerce},
  pages 243--251, NY, NY, USA, 2007. ACM.

\bibitem[Chawla et~al.(2010)Chawla, Hartline, Malec, and
  Sivan]{chawla2010multi}
Shuchi Chawla, Jason~D. Hartline, David~L. Malec, and Balasubramanian Sivan.
\newblock Multi-parameter mechanism design and sequential posted pricing.
\newblock In \emph{Proceedings of the Forty-second ACM Symposium on Theory of
  Computing}, pages 311--320, NY, NY, USA, 2010. ACM.

\bibitem[Cole and Roughgarden(2014)]{CR14}
Richard Cole and Tim Roughgarden.
\newblock The sample complexity of revenue maximization.
\newblock In \emph{Proceedings of the 46th Annual ACM Symposium on Theory of
  Computing}, pages 243--252, NY, NY, USA, 2014. SIAM.

\bibitem[Daniely and Shalev-Shwartz(2014)]{daniely2014multiclass}
Amit Daniely and Shai Shalev-Shwartz.
\newblock Optimal learners for multiclass problems.
\newblock In \emph{COLT 2014}, pages 287--316, 2014.
\newblock URL \url{http://arxiv.org/abs/1405.2420}.

\bibitem[Devanur et~al.(2015)Devanur, Huang, and Psomas]{devanurside2015}
Nikhil~R. Devanur, Zhiyi Huang, and Christos{-}Alexandros Psomas.
\newblock The sample complexity of auctions with side information.
\newblock \emph{CoRR}, abs/1511.02296, 2015.
\newblock URL \url{http://arxiv.org/abs/1511.02296}.

\bibitem[Dughmi et~al.(2014)Dughmi, Han, and Nisan]{dughmi2014sampling}
Shaddin Dughmi, Li~Han, and Noam Nisan.
\newblock Sampling and representation complexity of revenue maximization.
\newblock In \emph{Web and Internet Economics}, volume 8877 of \emph{Lecture
  Notes in Computer Science}, pages 277--291. Springer Intl. Publishing,
  Beijing, China, 2014.

\bibitem[Ehrenfeucht et~al.(1989)Ehrenfeucht, Haussler, Kearns, and
  Valiant]{ehrenfeucht1989general}
Andrzej Ehrenfeucht, David Haussler, Michael Kearns, and Leslie Valiant.
\newblock A general lower bound on the number of examples needed for learning.
\newblock \emph{Information and Computation}, 82\penalty0 (3):\penalty0
  247--261, 1989.
\newblock URL \url{https://www.cis.upenn.edu/~mkearns/papers/lower.pdf}.

\bibitem[Elkind(2007)]{elkind2007}
Edith Elkind.
\newblock Designing and learning optimal finite support auctions.
\newblock In \emph{Proceedings of the eighteenth annual ACM-SIAM symposium on
  Discrete algorithms}, pages 736--745. SIAM, 2007.

\bibitem[Feldman et~al.(2015)Feldman, Gravin, and
  Lucier]{feldman2015combinatorial}
Michal Feldman, Nick Gravin, and Brendan Lucier.
\newblock Combinatorial auctions via posted prices.
\newblock In \emph{Proceedings of the Twenty-Sixth Annual ACM-SIAM Symposium on
  Discrete Algorithms}, pages 123--135. SIAM, 2015.

\bibitem[Hanneke(2015)]{hanneke15}
Steve Hanneke.
\newblock The optimal sample complexity of {PAC} learning.
\newblock \emph{CoRR}, abs/1507.00473, 2015.
\newblock URL \url{http://arxiv.org/abs/1507.00473}.

\bibitem[Hartline and Roughgarden(2009)]{hartline2009simple}
Jason~D. Hartline and Tim Roughgarden.
\newblock Simple versus optimal mechanisms.
\newblock In \emph{ACM Conf. on Electronic Commerce}, Stanford, CA, USA., 2009.
  ACM.

\bibitem[Hsu et~al.(2016)Hsu, Morgenstern, Rogers, Roth, and Vohra]{walras2016}
Justin Hsu, Jamie Morgenstern, Ryan Rogers, Aaron Roth, and Rakesh Vohra.
\newblock Do prices coordinate markets?
\newblock In \emph{STOC}, page Forthcoming, 1 2016.

\bibitem[Huang et~al.(2015)Huang, Mansour, and Roughgarden]{huang2014making}
Zhiyi Huang, Yishay Mansour, and Tim Roughgarden.
\newblock Making the most of your samples.
\newblock In \emph{Proceedings of the Sixteenth ACM Conference on Economics and
  Computation}, EC '15, page forthcoming, New York, NY, USA, 2015. ACM.

\bibitem[Kelso~Jr and Crawford(1982)]{kelso1982}
Alexander~S Kelso~Jr and Vincent~P Crawford.
\newblock Job matching, coalition formation, and gross substitutes.
\newblock \emph{Econometrica: Journal of the Econometric Society}, pages
  1483--1504, 1982.

\bibitem[Littlestone and Warmuth(1986)]{littlestone1986compression}
Nick Littlestone and Manfred Warmuth.
\newblock Relating data compression and learnability.
\newblock Technical report, University of California, Santa Cruz, 1986.
\newblock URL \url{https://users.soe.ucsc.edu/~manfred/pubs/lrnk-olivier.pdf}.

\bibitem[Medina and Mohri(2014)]{medina2014learning}
Andres~Munoz Medina and Mehryar Mohri.
\newblock Learning theory and algorithms for revenue optimization in second
  price auctions with reserve.
\newblock In \emph{Proceedings of The 31st Intl. Conf. on Machine Learning},
  pages 262--270, 2014.

\bibitem[Morgenstern and Roughgarden(2015)]{morgenstern2015pseudo}
Jamie~H Morgenstern and Tim Roughgarden.
\newblock On the pseudo-dimension of nearly optimal auctions.
\newblock In \emph{Advances in Neural Information Processing Systems}, pages
  136--144, 2015.

\bibitem[Myerson(1981)]{myerson1981optimal}
Roger~B Myerson.
\newblock Optimal auction design.
\newblock \emph{Mathematics of operations research}, 6\penalty0 (1):\penalty0
  58--73, 1981.

\bibitem[Pollard(1984)]{pollard1984}
David Pollard.
\newblock \emph{Convergence of stochastic processes}.
\newblock David Pollard, New Haven, Connecticut, 1984.

\bibitem[Roughgarden and Schrijvers(2015)]{RS15}
T.~Roughgarden and O.~Schrijvers.
\newblock Ironing in the dark.
\newblock Submitted, 2015.

\bibitem[Rubinstein and Weinberg(2015)]{rubensteinsubadditive}
Aviad Rubinstein and S.~Matthew Weinberg.
\newblock Simple mechanisms for a subadditive buyer and applications to revenue
  monotonicity.
\newblock In \emph{Proceedings of the Sixteenth ACM Conference on Economics and
  Computation}, EC '15, pages 377--394, New York, NY, USA, 2015. ACM.

\bibitem[Vapnik and Chervonenkis(1971)]{VC}
Vladimir~N Vapnik and A~Ya Chervonenkis.
\newblock On the uniform convergence of relative frequencies of events to their
  probabilities.
\newblock \emph{Theory of Probability \& Its Applications}, 16\penalty0
  (2):\penalty0 264--280, 1971.

\bibitem[Vapnik and Kotz(1982)]{vapnik1982estimation}
Vladimir~Naumovich Vapnik and Samuel Kotz.
\newblock \emph{Estimation of dependences based on empirical data}.
\newblock Springer, 1982.

\bibitem[Yao(2015)]{yao2015reduction}
Andrew Chi-Chih Yao.
\newblock An n-to-1 bidder reduction for multi-item auctions and its
  applications.
\newblock In \emph{Proceedings of the Twenty-Sixth Annual ACM-SIAM Symposium on
  Discrete Algorithms}, pages 92--109. SIAM, 2015.

\end{thebibliography}

\appendix

\section{Open Problems}\label{sec:open}

We propose the following open problems resulting from our work.
\begin{enumerate}
\item Is it possible to construct ``compression-style'' arguments
  which bound the pseudo-dimension of the revenue of the class of item
  pricings for additive bidders which are tight (giving a bound of $k$
  and $nk$, as in Theorem~\ref{thm:tighter-additive}, rather than
  $k^2$ and $nk^2$)?
\item For general or even subadditive valuations, do item pricings
  have pseudo-dimension $O(nk)$ or strictly larger?
\item Is it possible to frame the allocations which result from item
  pricings with item-specific reserves as $n$-fold sequential
  allocation rules from some simple class, for general valuation
  functions? We were able to show it for additive valuations, which
  allowed us to use the ``trick'' where the highest bidder for an item
  is willing to pay anything less than her bid for that item
  (independent of other prices); thus, if she's willing to pay the
  reserve, by virtue of being the highest bidder for the item she's
  willing to pay the second-highest bid as well. For more general
  valuations, she may or may not optimize her utility by paying some
  combination of item prices and second-highest bids for a bundle
  which was utility-optimal if she were only paying item prices.
\item Relatedly, what is the pseudo-dimension of second-price item
  auctions with item-specific reserves when bidders have valuations
  which are more general than additive or unit-demand? One can use a
  proof similar to the proof of Theorem~\ref{thm:tighter-additive} to
  achieve a bound for unit-demand bidders, but what about for
  submodular or subadditive bidders? It isn't clear that the relative
  ordering of a small number of ``relevant'' parameters (such as
  per-item price and per-bidder single-item values) of the auction and
  sample are sufficient to fix the most-preferred bundle for each
  agent from a sample.
\end{enumerate}

\section{Binary Labeled Learning}\label{sec:binary}
Suppose there is some domain $\V$, and let $c$ be some unknown target
function $c: \V \to \{0,1\}$, and some unknown distribution $\D$ over
$\V$. We wish to understand how many labeled samples $(v, c(v))$, with
$v\sim \D$, are necessary and sufficient to be able to compute a
$\hat c$ which agrees with $c$ almost everywhere with respect to
$\D$. The distribution-independent sample complexity of learning $c$
depends fundamentally on the ``complexity'' of the set of binary
functions $\F$ from which we are choosing $\hat c$.  We review two
standard complexity measures next.

Let $N$ be a set of $\sam$ samples from $\V$. The set $N$ is said
to be \emph {shattered} by $\F$ if, for every subset $T\subseteq N$,
there is some $c_T\in\F$ such that $c_T(v) = 1$ if $v\in T$ and
$c_T(v') = 0$ if $v'\notin T$.  That is, ranging over all $c \in \F$
induces all $2^{|N|}$ possible projections onto $N$.  The {\em VC
  dimension} of $\F$, denoted $\VC(\F)$, is the size of the largest
set $S$ that can be shattered by $\F$.

Let $\err_N(\hat c) = (\sum_{v\in N} |c(v) - \hat{c}(v)|)/|N|$ denote
the empirical error of $\hat c$ on $N$, and let
$\err(\hat c) = \E_{v\sim \D}[|c(v) - \hat{c}(v)|]$ denote the
\emph{true} expected error of $\hat c$ with respect to $\D$.  We say
$\F$ is {\em $(\epsilon, \delta)$-\PAC learnable with sample
  complexity $\sam$} if there exists an algorithm $\A$ such that, for
all distributions $\D$ and all target functions $c\in\F$, when $\A$ is
given a sample $S$ of size $\sam$, it produces some $\hat{c}\in \F$
such that $\err(\hat{c}) < \epsilon$, with probability $1-\delta$ over
the choice of the sample.  The \PAC sample complexity of a class $\F$
can be bounded as a polynomial function of $\VC(\F)$, $\epsilon$, and
$\ln\frac{1}{\delta}$~\citep{VC}; furthermore, any algorithm which
$(\epsilon, \delta)$-PAC learns $\F$ over all distributions $\D$
\emph{must} use nearly as many samples to do so. The following theorem
states this well-known result formally.\footnote{The upper bound
  stated here is a quite recent result which removes a
  $\ln\frac{1}{\epsilon}$ factor from the upper bound; a slightly
  weaker but long-standing upper bound can be attributed
  to~\citet{vapnik1982estimation}.}

\begin{theorem}[{Upper bound~\citep{hanneke15}, Lower bound, Corollary
    5 of~\citep{ehrenfeucht1989general}}]
  \label{thm:pac-convergence}
  Suppose $\F$ is a class of binary functions. Then, $\F$ can be
  $(\epsilon, \delta)$-PAC learned with a sample of size
\[\sam  = O\left(\frac{\VC(\F) + \ln\frac{1}{\delta}}{\epsilon}\right).\]
Furthermore, any $(\epsilon, \delta)$-\PAC learning algorithm for $\F$
must have sample complexity
\[\sam = \Omega\left(\frac{\VC(\F) + \ln\frac{1}{\delta}}{\epsilon}\right).\]
\end{theorem}

There is a stronger sense in which a class $\F$ can be learned, called
\emph{uniform learnability}. This property implies that, with a
sufficiently large sample, the error of \emph{every} $c\in\F$ on the
sample is close to the true error of $c$.  We say $\F$ is {\em
  $(\epsilon, \delta)$-uniformly learnable with sample complexity
  $\sam$} if, for every distributions $\D$, given a sample $N$ of size
$\sam$, with probability $1-\delta$,
$|\err_N(c) - \err(c)| < \epsilon$ for every $c\in \F$.
Notice that, if $\F$ is $(\epsilon, \delta)$-uniformly learnable with
$\sam$ samples, then it is also $(\epsilon, \delta)$-\PAC learnable
with $\sam$ samples.  We now state a well-known upper bound on the
uniform sample complexity of a class as a function of its VC
dimension.

\begin{theorem}[See, e.g.~\citet{VC}]
\label{thm:uniform-convergence}
  Suppose $\F$ is a class of binary functions. Then, $\F$ can be
  $(\epsilon, \delta)$-uniformly learned with a sample of size
\[\sam  = O\left(\frac{\VC(\F)\ln\frac{1}{\epsilon} + \ln\frac{1}{\delta}}{\epsilon^2}\right).\]
\end{theorem}

\section{Formal Statements of Known Revenue Guarantees for Simple Mechanisms}\label{sec:formal-known}

In various special cases, it has been shown that the aforementioned
auctions earn a constant fraction of the optimal revenue. All of these
results rely on buyers' valuations displaying some kind of
independence across items: for additive and unit-demand buyers, this
just means that for all $i$,
$v_{i} = (v_{i1}, \ldots, v_{ik}) \sim \D = \D_1 \times \ldots \D_k$
is drawn from a product distribution.  Under this
assumption,~\citet{chawla2010multi} showed that individualized item
pricings are sufficient to earn a constant fraction of optimal revenue.
\begin{theorem}\label{thm:unit-rep}[\citet{chawla2010multi}]
  Suppose each $i\in [\nbuy]$ has a unit-demand valuation
  $v_i\sim \D_i = (\D_{i1}\times \ldots \times \D_{ik})$. Then, there
  exists some nonanonymous item pricing $p\in\R^{k \nbuy}$ such that
\[\rev(p, \D) \geq \frac{1}{10.67}\rev(\opt).\]
\end{theorem}

For a single item-independent additive buyer, the better of the best
item pricing and grand bundle pricing also earns a constant fraction
of optimal revenue for that setting~\citep{babaioffadditive}.

\begin{theorem}\label{thm:add-rep}[\citet{babaioffadditive}]
  Suppose there is a single buyer which has an additive valuation
  $v_i\sim \D_i = (\D_{i1}\times \ldots \times \D_{ik})$. Then, for an item pricing $p\in\R^{k}$ and $q$ a grand bundle price,  $q\in\R$
\[\max(\max_{p\in\R^{k}}\rev(p, \D_i),\max_{q\in \R}\rev(q, \D_i)) \geq \frac{1}{6}\rev(\opt, \D_i).\]
\end{theorem}

A recent result of~\citet{yao2015reduction} showed that one can reduce
the problem of designing approximately optimal mechanisms for $n$
additive buyers to the problem of designing approximately optimal
mechanisms for each single additive buyers, subject to selling each
item to the highest bidder for that item (while losing a constant
factor in terms of revenue). When combined with the aforementioned
result for a single additive bidder, this implies that the best of
second price item auctions with the best individualized item reserves
and second price grand bundle auctions with the best individualized
bundle reserve, is also approximately revenue-optimal for $n$
(non-identically distributed) buyers with valuations which are
independent across items.

\begin{theorem}\label{thm:add-n-rep}[Applying~\citet{yao2015reduction}
to~\citet{babaioffadditive}] Suppose each buyer $i\in[\nbuy]$ has an
additive valuation
$v_i\sim \D_i = (\D_{i1}\times \ldots \times \D_{ik})$.
Let $s_p$ for $p\in\R^{k\nbuy}$ represent the second-price item
auction with reserve $p_{ij}$ for buyer $i$ and item $j$ (and
similarly, let $s_q$ for $q\in\R^\nbuy$ represent the second-price
grand bundle auction with reserve $q_i$ for buyer $i$). Then,

\[\max(\max_{p\in\R^{k\nbuy}}\rev(s_p, \D_i),\max_{q\in \R^\nbuy}\rev(s_q, \D_i)) \geq \frac{1}{8}\rev(\opt, \D_i).\]
\end{theorem}

The final well-known result for approximately
optimal~\citet{rubensteinsubadditive}, ``simple'' revenue-maximizing
mechanisms states that, for an appropriately generalized definition of
valuations distributed ``independently across items'', one can
approximately maximize revenue selling to a single subadditive buyer
with item or grand bundle pricings. We now present the formal
definition of independence they use for these more complicated
valuation functions, and present their main result.

\begin{definition}[\citet{rubensteinsubadditive}]
  A distribution $\D$ over valuation functions $v : 2^k \to \R$ is
  subadditive over independent items if:
\begin{enumerate}
\item All $v$ in the support of $\D$ are monotone;
  $v(\K\cup\K') \geq v(\K)$ for all $\K,\K'$.
\item All $v$ in the support of $\D$ are subadditive:
  $v(\K \cup \K') \leq v(\K) + v(\K')$ for all $\K, \K'$.
\item All $v$ in the support of $\D$ exhibit no externalities: there
  exists some $\D^{\vec x}$ over $\R^k$ and a function $V$ such that
  $\D$ is a distribution that samples $\vec x \sim \D^{\vec x}$ and
  outputs $v$ such that $v(\K) = V(\{x_\k\}_{\k\in \K}, \K)$ for all $\K$.
\item $\D^{\vec x}$ is product across its $k$ dimensions.
\end{enumerate}
\end{definition}

\begin{theorem}\label{thm:sub-rep}[\citet{rubensteinsubadditive}]
  Suppose $\D_i$ is subadditive over independent items. Then, there
  exists a universal constant $c\geq 1$ such that
\[\max(\max_{p\in\R^{k}}\rev(p, \D_i),\max_{q\in \R}\rev(q, \D_i)) \geq \frac{1}{c}\rev(\opt, \D_i).\]
\end{theorem}

\section{A tighter bound on the pseudo-dimension of second-price item auctions with reserves for additive bidders}\label{sec:additive}

We now present a tighter analysis of second-price item auctions with
reserves which exploits the total independence of buyers' behavior on
items $j, j'$.

\begin{theorem}\label{thm:tighter-additive}
  The pseudo-dimension of item auctions and second-price item auctions
  with anonymous item reserves is $O(k\log k)$ and with nonanonymous
  item prices/reserves is $O(nk\log(nk))$ when bidders are additive.
\end{theorem}

\begin{proof}
  We present the proof for the class of second-price item auctions
  with item reserves; the item price result follows easily since the
  winner for $j$ always pays her item price (rather than the maximum
  of that and the second-highest bid for $j$).

  Rather than proving the allocation rules are linearly separable, we
  upper-bound the number of intermediate labelings these classes can
  induce for $m$ samples, where the intermediate label space we
  consider is the allocation combined with, for each item, whether the
  winner for that item is paying the item's reserve or second price
  for that item. Fix some sample $S = (\v^1, \ldots, \v^m)$ where
  $v^t\in\V^n$ and $(r^1, \ldots, r^m)\in \R^m$.

  This can be encoded in $\{0,1\}^{2k}$ for anonymous item reserves (a
  bit for whether or not an item is sold at its reserve and another
  for whether it is sold for its second-price), and $\{0,1\}^{2nk}$
  for nonanonymous reserves (where each item is labeled as being
  allocated to some bidder, along with whether it is sold for that
  bidder's item-specific reserve or the second-highest bid). In the
  latter case, there is a post-processing rule which can reduce the
  label space to have size $O(n^{2k})$, since all allocations are
  feasible allocations. In both cases, we will use $y^t$ to denote the
  intermediate label for sample $\v^t$.

  We begin with anonymous item reserves. Since buyers are additive, we
  can consider each item separately. We consider item $j\in[k]$. There
  are $2m+1$ relevant quantities which affect the revenue any reserve
  achieves for item $j$: $p_j$, the reserve for $j$, and for each
  $t\in[m]$, $v^t_{i^*_j}(\{j\})$ and $v^t_{i'_j}(\{ j\})$, where
  $i^*_j, i'_j$ are the first and second highest bidders for $j$ from
  sample $t$, respectively. When $p_j \leq v^t_{i'_j}(\{ j\})$, let
  $y^t_{j} = 1$ and $y^t_{k + j} =0 $, when
  $v^t_{i^*_j}(\{j\}) \geq p_j > v^t_{i'_j}(\{ j\})$, let
  $y^t_{j} = 0$ and $y^t_{k + j} =1 $, and when
  $p_j > v^t_{i^*_j}(\{ j\})$, let $y^t_j = y^t_{n+j} = 0$. Thus, when
  the relative ordering of these $2m+1$ parameters is fixed, the $j$th
  and $n+j$th coordinates for all $m$ samples are fixed. Varying $p_j$
  can induce at most $2m+2$ distinct labelings of all of $S$. Thus,
  for all $k$ items, there are at most $(2m+2)^k$ distinct vectors
  $(y^1, \ldots, y^t)$. 

  Now, fix some intermediate labeling $(y^1, \ldots, y^m)$ of
  $S$. Then, the revenue for a particular reserve vector
  $(p_1, \ldots, p_k)$ which induces this labeling on the sample is
  easy to describe as a linear of this labeling. Namely,
  \[rev(\v^t, \p, y^t) = \sum_{j : y^t_j = 1} v^t_{i'_j}(\{j\}) +
  \sum_{j : y^t_{n+j} = 1} \p_j \]
  which is a linear function in $2k$ dimensions of $\p$ and
  $\v^t, y^t$ (which are constants). Thus, since linear functions in
  $2k$ dimensions have VC-dimension $2k+1$, the item reserves which
  agree with $(y^1, \ldots, y^m)$ can induce at most $m^{2k+1}$
  labelings of $S$ with respect to $(r^1, \ldots, r^m)$.

  Thus, the set of all item reserves can induce at most
  $m^{2k+1} \cdot (2m+2)^{k}$ labelings with respect to
  $(r^1, \ldots, r^m)$, so if $S$ is shatterable it must be that $2^m
  \leq m^{2k+1} \cdot (2m+2)^{k}$, or that $m = O(k\log k)$. 

  With nonanonymous reserves, each sample will instead be given an
  intermediate label in $\{0,1\}^{nk + k}$, where there is a bit for
  each item/bidder pair (corresponding to whether or not that bidder
  wins the item and pays her individualized reserve for the item), and
  an additional bit for each item (corresponding to whether or not
  that item is sold for its second-highest bid). There are at most
  $[n+1]^k$ valid labelings of a single sample (each item is sold to
  at most one bidder, and is either sold to her at her reserve or at
  the second-highest price). For $m$ samples, for a particular item
  $j$, there are now $2m + n$ parameters whose ordering matters (the
  highest and second-highest bids and the bidder-specific reserves for
  that item); the bidder-specific item reserves for that item can
  induce at most $(2m + n)^n$ distinct orderings of these parameters;
  fixing this ordering, the intermediate label is also fixed for all
  samples. Furthermore, once one has fixed the intermediate label for
  all samples, the revenues of all individualized item reserve
  auctions which agree with that intermediate labeling are again
  expressible as a linear function in $2nk$ dimensions. Thus, if the
  sample is shatterable, $2^m \leq (2m + n)^{nk} \cdot m^{2nk}$,
  implying $m = O(nk\ln(nk))$.
\end{proof}

\section{Omitted Proofs}\label{sec:omitted}

\begin{prevproof}{Remark}{rem:ind-labels}
  For each $x\in X, y\in Q_x$, there exists $\Psi(x,y)$ and for
  $f\in \F_1$, some $w^f\in \R^d$ such that
  \[\argmax_{y\in Q_x} \Psi(x,y) \cdot w^f = f(x).\]
  We simply must extend the definition of $\Psi(x,y)$ to be defined
  over all $y\in Q$
\[\Psi(x,y') \cdot w^f < \max_{y\in Q_x} \Psi(x,y) \cdot w^f\]
for $y' \in Q \setminus Q_x$. Define $\Psi(x, y')_{t} = 0$ for any
$t\notin T^+$, and $\Psi(x,y')_{t} = -1$ for all $t\in T^+$. Then, for
any $y' \in Q\setminus Q_x$, the dot product
$\Psi(x,y') \cdot w^f < 0 \geq \max_{y\in Q_x} \Psi(x,y) \cdot w^f $,
so the maximizing label $y$ will still be in $Q_x$.
\end{prevproof}

\begin{prevproof}{Theorem}{thm:linsep}
  Consider a sample $S= (x^1, \ldots, x^m)\in \X^m$ of size $m$ with
  targets $r = (r^1, \ldots, r^m)\in\R^m$. We first claim that, since
  $\F_1$ is $a$-dimensionally linearly separable, $\F_1$ can label $S$
  in at most ${m \choose a} \cdot |Q|^a$ distinct
  ways. Theorem~\ref{thm:linsep-compression} implies that $\F_1$ must
  admit a compression scheme $\com, \decom$ of size at most $a$.  Let
  $f_1(S)$ denote the labeling of all of $S$ by some fixed
  $f_1\in\F_1$.  Then, $\F_1$ can label $S$ in at most
  $|\textrm{range}_{f_1\in\F_1}(\decom \circ \com)(S, f_1(S))|$ ways
  since this is a compression scheme for $\F_1$.  The decompression
  function takes as input $a$ \emph{labeled} examples which are a
  subset of $S$, so it will have one of ${m \choose a} \cdot |Q|^a$
  inputs for a fixed set $S$ (some subset of $S$ labeled in some
  arbitrary way), and therefore at most that many outputs, which
  upper-bounds the total number of possible labelings of $S$ by the
  same quantity.

 Then, fixing the labeling of $S$ to
  be consistent with some $f_1 \in \F_1$, we know that the
  pseudo-dimension of $\F_{2|f_1}$ is at most $b$, so it can induce at
  most $m^b$ many labelings of $S$ according to $r$. Thus, there are
  at most $m^a |Q|^a m^b$ binary labelings of $S$ with respect to $r$
  over all of $\F_2$ (and, therefore over all of $\F$). If $S$ is
  shatterable, it must be that
\[2^m \leq m^a |Q|^a m^b\]
implying $m \leq (a + b)\ln(m) + a\ln|Q|$, as desired.
\end{prevproof}

\begin{prevproof}{Theorem}{thm:sequential}
  In either case, $Q$ is a set of feasible allocations, so we only
  must show linear separability over the set of feasible allocations
  (that is, we need only show separability over labels
  $\B : \B_i \cap \B_j = \emptyset$).

  We start with the first case of sequential allocations.  We will
  show that $\F_1$ is $an$-dimensionally linearly separable. By
  definition, $\F_1$ is a set of $n$-wise sequential allocations from
  some $\H$ which is $a$-dimensionally linearly separable over
  $\{0,1\}^k$.  This means there exists some
  $\Psi : (\V \times \{0,1\}^k) \times \{0,1\}^k \to \R^a, w^{h} \in
  \R^d$
  such that $\argmax_{B}\Psi((v,X), B) \cdot w^h = h(v, X)$ for all
  $h\in \H, (v, X) \in \V\times \{0,1\}^k$.

  We simply need to construct some new
  $\hat{\Psi} : \V^n \times Q \to \R^{an}, \hat{w}^{h_1, \ldots, h_n}
  \in \R^{an}$ such that
  \[\argmax_{\B = (B_1 \ldots B_n)}\hat{\Psi}(\v, \B) \cdot
  \hat{w}^{h_1, \ldots, h_n} = (h_1(\v_1, X_1(v)), h_2(\v_2, X_2(\v)), \ldots,
  h_n(\v_n, X_n(\v))).\]
  Define $\alpha_i = 2^i \maxval$, and define
\[
 \Psi((\v, \B)_{ij} =
 \alpha_i  \cdot \Psi((\v_i, [k] \setminus \cup_{i' < i} \B_{i'}) \B_i)_j   
\]
Then, for some $\prod_{(h_1, \ldots, h_n)}\in \F_1$, define
\[
\hat{w}^{h_1, \ldots, h_n}_{ij} =
     w^{h_i}_j
\]
Now, inspecting the dot product for some $\v, \B$ we see
\[\hat{\Psi}(\v, \B) \cdot \hat{w}^{h_1, \ldots, h_n} = \sum_{i} \alpha_i  \Psi((\v_i, [k]\setminus \cup_{i' < i} \B_{i'}), \B_i) \cdot w^{h_i}\]
which, by the definition of $\alpha_i$ and the assumption that
$\Psi((\v_i, X), B) \cdot w^h \leq \maxval$ for all
$\v_i, X, B, h\in\H$ implies that the maximizing label $\B $ will
first pick $\B_1\subseteq [k] = X_1(\v)$ to maximize
$\Psi((\v_1, X_1(\v)), \B_1) \cdot w^{h_1}$, then will pick
$\B_2 \subseteq [k]\setminus \B_1 = X_2(\v)$ to maximize
$\Psi((\v_2, X_2(\v)), \B_2) \cdot w^{h_2}$, and so on.  Thus, $\F_1$
is $an$-dimensionally linearly separable.

Now, suppose $\F_1$ is a set of $n$-wise repeated allocations. Since
$\H$ is $a$-dimensionally linearly separable, we know that for all
$v, X, B_i$, there exists $\Psi((v_i, X), B_i)$, and for all $h\in\H$ there
is some $w^h$ such that
\[\argmax_{B_i} \Psi((v_i, X), B_i) \cdot w^h = h(v_i, X).\]
We simply need to define some $\hat{\Psi} : \V^n \times Q \to \R^a$,
$\hat{w}^h\in \R^a$ such that
\[\argmax_{\B}\hat{\Psi}(\v, \B) \cdot \hat{w}^h = (h(\v_1, X_1(\v)), h(\v_2, X_2(\v)), \ldots, h(\v_n, X_n(\v))). \]
Then, define
\[
 \hat{\Psi}(\v,\B)_{x} =
 \sum_i \alpha_i  \cdot \Psi((\v_i, [k] \setminus \cup_{i' < i} \B_{i'}), \B_i)_x  
\]
Then, for some $\prod_{h, \ldots, h)}\in \F_1$, define
\[
\hat{w}^{h}_{x} =
     w^{h}_x
\]
Then, the dot product 
\[\Psi(\v, \B) \cdot \hat{w}^{h} = \sum_{x} \sum_i \alpha_i  \cdot \Psi((\v_i, [k] \setminus \cup_{i' < i} \B_{i'}), \B_i)_x   \cdot  w^{h}_x =  \sum_i \alpha_i  \cdot \Psi((\v_i, [k] \setminus \cup_{i' < i} \B_{i'}), \B_i)   \cdot  w^{h} ,\]
which by the definition of $\alpha_i$ and the guaranteed upper bound
on the dot product $\Psi \cdot w^h \leq \maxval$, we know will be
maximized by first picking some $\B_1\subseteq[k] = X_1(\v)$ which
maximizes $\Psi((\v_1, X_1(\v)), \B_1) \cdot w^h$, then picking
$B_2\subseteq X_1(\v) \setminus h(\v_1, X_1(\v)) = X_2(\v)$ which
maximizes $\Psi((\v_2, X_2(\v)),\ B_2) \cdot w^h$, and so on. Thus, $\F_1$
is $a$-dimensionally linearly separable.
\end{prevproof}

\begin{prevproof}{Corollary}{cor:bundle}
  We first prove first that for a single buyer, the grand-bundle
  mechanism is $2$-dimensionally linearly separable over $\{0,1\}$.  Let $\H$ denote
  the class of single-buyer grand bundle pricings. For some $h\in\H$,
  we define $h_1 : \V \to \{0,1\}$ as $h_1(v) = \I[v \geq p^h]$, where
  $p^h\in\R$ represents the price of the grand bundle under $h$. We
  will show $\H$ is $2$-dimensionally linearly separable over
  $\{0,1\}$. Define $\Psi(v, b) = \I[b = 1] (v([k]), 1)$ for each
  $b\in\{0,1\}$ and $w^h = (1, -p^h)$. Then,
\[\argmax_b \Psi(v, b) \cdot w^h = \argmax_b  \I[b = 1] (v([k] - p^f) =  \I[v \geq p^h] = f_1(v)\]
since the penultimate expression is maximized by $b= 1$ only if
$v([k]) \geq p^f$. Thus, $\H$ is $2$-dimensionally separable.
 
Notice that when $\F$ is the set of anonymous grand bundle pricings,
its allocation rules $\F_1$ are $n$-fold repeated allocations from
$\H$.  Thus, by Theorem~\ref{thm:sequential}, anonymous grand bundle
pricings' allocations are linearly separable in $2$ dimensions. The
obvious intermediate label space
$Q = \{\vec{0}\}\cup \{e_i | i \in [n]\}$, the set of standard basis
vectors, contains more information than is needed to compute the
revenue of these auctions. Define $q(x) = \I[||x|| > 0]$;
Observation~\ref{obs:postprocessing} implies that $\F'_1 = q \circ \F_1$ is
$2$-dimensionally linearly separable over $Q' = \{0,1\}$.  Now we prove
for each $f_1\in\F'_1$ that $\F_{2|f_1}$ has pseudo-dimension
$O(1)$. Fix some $f_1\in\F'_1$. Then, we have that
\[f'_2(\v) = f_2(f_1(\v), \v) = p^f \cdot f_1(\v)\]
so, the class $\F_{2|f_1}$ is a class of linear functions in $1$
dimensions, which have pseudo-dimension at most $2$. Thus, $\F$ is
$(2,2)$-factorable over $\{0,1\}$, and Theorem~\ref{thm:linsep} implies
that the pseudo-dimension of anonymous grand bundle pricings is
$O(1)$.

Similarly, when $\F$ is the the set of non-anonymous grand bundle
pricings, $\F_1$ are $n$-fold sequential allocations from
$\H$. Thus, these allocation rules are $2n$-dimensionally linearly
separable, respectively. In this case, we leave the intermediate label
space as $Q = \{\vec{0}\}\cup \{e_i | i \in [n]\}$. For any $f\in \F$,
let $p^f\in\R^n$ denote the price vector for the grand bundle, that is
$p^f_i$ is $i$'s price for purchasing the grand bundle. Fix some
$f_1\in \F_1$; we claim that $\F_{2|f_1}$ has pseudo-dimension $O(n)$.
For any $f'_2\in \F_2$ and any $f$ which is decomposed into
$(f_1, f_2)$ , we have that
$f'_2(\v) = f_2(f_1(\v), \v) = p^f \cdot f_1(v) $, which again is a
linear function when $f_1$ is fixed, in this case in $n$ dimensions.
Thus, $\F$ is $(2n, n)$-factorable over $Q$, so
Theorem~\ref{thm:linsep} implies that the pseudo-dimension of
nonanonymous grand bundle pricings is $O(n\log n)$.
\end{prevproof}

\begin{prevproof}{Corollary}{cor:items}
  As in the previous proof, we claim that $\F_1$, the allocation rules
  of these auctions are $n$-wise repeated allocations and $n$-wise
  sequential allocations from the single-buyer item pricings
  allocation set $\H$. We begin by showing $\H$ is
  $k+1$-dimensionally linearly separable. For some $h\in \H$, let
  $p^h\in \R^n$ denote the item pricing faced by the single buyer. Then, 
define for $v\in\V$, $B\in \{0,1\}^k$,
\[
 \Psi(v, B)_{j} =
 \begin{cases} 
      \hfill \I[j\in B]   \hfill & \text{ if $j\in [k]$} \\
        \hfill v(B) \hfill & \text{ if $j  = k+1$ } \\
  \end{cases}
\]
and for $h\in\H$, define 
\[
w^h_j =
 \begin{cases} 
      \hfill -p^h_j   \hfill & \text{ if $j\in [k]$} \\
        \hfill 1 \hfill & \text{ if $j  = k+1$ }. \\
  \end{cases}
\]
Then, we have that $\Psi(v, B)\cdot w^h = v(B) - \sum_{j\in B}p^h_j$,
which will be maximized by $B$ which maximizes $v$'s utility. Thus,
$h(v) = \argmax_B v(B) - \sum_{j\in B}p^h_j = \argmax_B \Psi(v,
B)\cdot w^h$, so $\H$ is $k+1$-dimensionally linearly separable.

Consider $\F$ the class of anonymous item prices.
Theorem~\ref{thm:sequential} implies that this class is
$k+1$-dimensionally linearly separable over $Q = [n]^k$.  Again, the
intermediate label space suggested by this reduction to the
single-buyer case, $Q = [n]^k$, is larger than necessary to compute
revenue. We define $q(\B)_j = \I[ j \in \cup_i \B_i ]$, and by
Observation~\ref{obs:postprocessing}, $\F'_1 = q \circ \F_1$ is
$k+1$-dimensionally linearly separable over $Q' = \{0,1\}^k$.  We now
show that, for a fixed $f_1\in \F_1$, the class $\F_{2|f_1}$ has
pseudo-dimension $O(k)$. Notice that for any $f'_2 \in \F_{2|f_1}$, we
have that
\[f'_2(\v) = f_2(f_1(\v),\v) = p^f \cdot f_1(v)\]
which, again is a $k$-dimensional linear function for some fixed
$f_1$, and therefore has pseudo-dimension at most $k+1$. Thus, the
class $\F$ can be $(k+1, k+1)$-factored over $\{0,1\}^k$, and
Thereof~\ref{thm:linsep} implies the pseudo-dimension is thus at most
$O(k^2)$.

The proof for the nonanonymous case is identical, with two changes.
First, $\F_1$ is a set of $n$-wise \emph{nonanonymous} sequential
allocations, so it is linearly separable in $n(k+1)$ dimensions.
Second, we cannot compress the intermediate label space
$Q \subset\{0,1\}^{nk}, |Q| \leq [n]^k$, since
$f'_2(\v) = f_2(f_1(\v), \v) = p^f \cdot f_1(v)$ only expresses the
revenue of the auction if $f_1(v)$ expresses which buyers purchase
which items; thus, the set $\F_{2|f_1}$ has pseudo-dimension at most
$O(nk)$. Thus, the class $\F$ can be $(O(nk), O(nk))$-factored over
$Q$ with $|Q| \leq [n]^k$, and Thereof~\ref{thm:linsep} implies the
pseudo-dimension is thus at most $O( nk^2\ln(n))$.
\end{prevproof}

\end{document}